%% file: main.tex
\documentclass[twocolumn,10pt]{article}[2019/10/25]
\usepackage{graphicx} 
\RequirePackage{geometry}[2018/04/16]
\usepackage{hyperref}
\usepackage{amsmath}
\usepackage{amssymb}
\usepackage[mathscr]{euscript}
\usepackage{tikz}
\usetikzlibrary{cd,positioning,calc,decorations.pathmorphing,shapes}

\usepackage{titling}
\usepackage{natbib}
\usepackage{mathtools}
\usepackage{booktabs}
\usepackage{tikz}
\usepackage{adjustbox}
\usepackage[linesnumbered,ruled]{algorithm2e}

\renewenvironment{abstract}{
    \begin{adjustbox}{minipage=\linewidth-0.5in,center}
    \centerline{\large\bfseries Abstract}
    \vspace{2\baselineskip}
}{
    \end{adjustbox}
}
\RequirePackage[inline]{enumitem}[2019/06/20]
    \setlist[1,2,3,4,5,6]{
        labelsep=5pt,
        leftmargin=2em,
        topsep=4pt plus 1pt minus 2pt,
        partopsep=1pt plus 0.5pt minus 0.5pt,
        itemsep=2pt plus 1pt minus 0.5pt,
        parsep=2pt plus 1pt minus 0.5pt
    }
    \setlist[2,3,4,5,6]{
        topsep=2pt plus 1pt minus 0.5pt,
        itemsep=1pt plus 0.5pt minus 0.5pt,
        parsep=1pt plus 0.5pt minus 0.5pt,
    }
    \setlist[3,4,5,6]{
        leftmargin=1.5em,
        topsep=1pt plus 0.5pt minus 0.5pt,
        partopsep=0.5pt plus 0pt minus 0.5pt,
        parsep=0pt
    }
    \setlist[4,5,6]{leftmargin=1em}
    \setlist[5,6]{leftmargin=0.5em}
    \setlist[6]{leftmargin=0.5em} 
\RequirePackage{microtype}[2019/11/18]
\RequirePackage{courier}[2005/04/12]

\RequirePackage[overload]{textcase}
\RequirePackage[bf,raggedright,small,pagestyles]{titlesec}[2019/10/16]
    \titleformat*{\section}{\raggedright\large\bfseries\MakeUppercase}
    \titleformat*{\subsection}{\raggedright\bfseries\MakeUppercase}
    \titlespacing{\section}{0pt}{*2.75}{*1.375}
    \titlespacing{\subsection}{0pt}{*2.75}{*1.375}
    \titlespacing{\paragraph}{0pt}{*1.375}{1em}

\newcommand{\gitlink}[0]{https://github.com/WillemSch/Causal-Embeddings-and-the-Multi-Resolution-Marginal-Problem}

\input{Commands}

\title{Multi-Level Causal Embeddings}
\author{Willem Schooltink\\ \href{mailto:<willem.schooltink@uib.no>?Subject=Multi-Level Causal Embeddings}{\smaller willem.schooltink@uib.no} \and Fabio Massimo Zennaro\\ \href{mailto:<fabio.zennaro@uib.no>?Subject=Multi-Level Causal Embeddings}{\smaller fabio.zennaro@uib.no}} 
\date{\vspace{-.5em}Department of Informatics, University of Bergen
\vspace{3em}}

\begin{document}
\maketitle

\begin{abstract}

\input{Sections/0-Abstract}
\end{abstract}

\input{Sections/1-Introduction}

\input{Sections/1.5-Related-Works}

\input{Sections/2-Preliminaries}

\input{Sections/3-Embeddings}

\input{Sections/4-Marginal-Problem}

\input{Sections/5-Discussion}

\vspace{1em}
\textbf{Acknowledgements}

We thank Adèle Ribeiro for helping in highlighting the projection mechanism, allowing for neater definitions and illustrating a close connection between this work and CDAGs.

\bibliography{refs}
\newpage

\onecolumn
\title{Multi-Level Causal Embeddings\\(Supplementary Material)}
\author{}
\date{}
\maketitle

\appendix
\input{Sections/A-Appendix}
\end{document}

%% file: Commands.tex
\usepackage{amsmath}
\usepackage{amssymb, amsthm}
\usepackage{times}
\usepackage[mathscr]{euscript}
\usepackage{tikz}
\usetikzlibrary{cd,positioning,calc,decorations.pathmorphing,shapes}

\newcommand{\model}[0]{\ensuremath{\mathcal{M}}}
\newcommand{\boldalpha}[0]{\ensuremath{\boldsymbol{\alpha}}}

\newcommand{\relevant}[0]{\ensuremath{\mathbf{R}}}
\newcommand{\relevanttarget}[0]{\ensuremath{\mathbf{S}}}

\newcommand{\vars}[0]{\ensuremath{\mathbf{V}}}

\newcommand{\exos}[0]{\ensuremath{\mathbf{U}}}

\newcommand{\dists}[0]{\ensuremath{P(\exos)}}
\newcommand{\funcs}[0]{\ensuremath{\mathcal{F}}}
\newcommand{\range}[0]{\ensuremath{\mathcal{R}}}

\newcommand{\lone}[0]{\ensuremath{\mathcal{L}_1}}
\newcommand{\ltwo}[0]{\ensuremath{\mathcal{L}_2}}
\newcommand{\li}[0]{\ensuremath{\mathcal{L}_i}}
\newcommand{\alphabs}[0]{$\boldsymbol{\alpha}$-abstraction}\newcommand{\xdasharrow}[2][->]{
    \tikz[baseline=-\the\dimexpr\fontdimen22\textfont2\relax]{
        \node[anchor=south,font=\scriptsize, inner ysep=1.5pt,outer xsep=2.2pt](x){#2};
        \draw[shorten <=3.4pt,shorten >=3.4pt,dashed,#1](x.south west)--(x.south east);
    }
}
\newcommand{\confarrow}[0]{\xdasharrow[<->]{\;\;\;\;\;}}

\makeatletter
\providecommand{\leftsquigarrow}{%
  \mathrel{\mathpalette\reflect@squig\relax}%
}
\newcommand{\reflect@squig}[2]{%
  \reflectbox{$\m@th#1\rightsquigarrow$}%
}
\makeatother

\newcounter{sarrow}
\newcommand\medconfarrow[0]{%
\stepcounter{sarrow}%
\mathrel{\begin{tikzpicture}[baseline= {( $ (current bounding box.south) + (0,-0.5ex) $ )}]
\node[inner sep=.5ex] (\thesarrow) {$\;\;\;$};
\path[draw,<->,decorate,
  decoration={zigzag,amplitude=0.7pt,segment length=1.2mm,pre=lineto,pre length=4pt, post length=4pt}] 
    (\thesarrow.south east) -- (\thesarrow.south west);
\end{tikzpicture}}%
}

\newtheorem{theorem}{Theorem}
\newtheorem{lemma}[theorem]{Lemma}

\newtheorem{definition}{Definition}
\newtheorem{example}{Example}
\newtheorem{remark}{Remark}

\usepackage{array}   
\newcolumntype{L}{>{$}l<{$}} 
\newcolumntype{R}{>{$}r<{$}} 
\newcolumntype{C}{>{$}c<{$}} 

%% file: Sections/0-Abstract.tex
Abstractions of causal models allow for the coarsening of models such that relations of cause and effect are preserved. Whereas abstractions focus on the relation between two models, in this paper we study a framework for \textit{causal embeddings} which enable multiple detailed models to be mapped into sub-systems of a coarser causal model. We define causal embeddings as a generalization of abstraction, and present a generalized notion of consistency. By defining a \textit{multi-resolution marginal problem}, we showcase the relevance of causal embeddings for both the statistical marginal problem and the causal marginal problem; furthermore, we illustrate its practical use in merging datasets coming from models with different representations. 

%% file: Sections/1-Introduction.tex
\section{Introduction}
Causality enables us to reason about real-world systems on a level beyond statistics, allowing us to answer questions on the effects of interventions and hypotheticals. Such queries naturally show in many fields, such as medicine, biology or economics, where studying correlation is not enough. The formalism of Structural Causal Models (SCMs) \citep{pearl2000models} rigorously captures reasoning about observations, interventions and counterfactuals (hypotheticals). 

However, causal models of real-world systems often grow very large to a point where reasoning becomes impractical, as SCMs, based on Directed Acyclic Graphs (DAGs), do not scale very well. One solution to the issue is to work with models at a coarser resolution. For example, when modeling the wildlife populations in a forest, we may have data on all subspecies of deer and rodents present, but for our purposes we may only care about the combined population of all deer subspecies. In such cases we can use \textit{causal abstraction} to describe how a detailed (low-level) model maps to a coarser (high-level) model, whilst preserving causal relations when merging variables and the values they take. Such frameworks of causal abstraction \citep{rischel2020category,Beckers_Halpern_2019,beckers2020approximate} have rigorous mathematical foundations, and provide measures to evaluate consistency among the models. 

Another solution to deal with large SCMs is to reason about sub-systems. In such cases we wish to map, or \textit{embed}, the detailed models into sub-parts of a coarse model. Fig.\ref{fig:visual-emb-vs-abs} illustrates how embeddings compare against abstractions. While abstraction deals with one-to-one mappings of models, embeddings deal with the common challenge in the sciences of having a high-level model of a system, for example an overarching climate model, and a combination of low-level sub-models, each describing only a part of the global model.

\begin{figure}
    \centering
    \input{Figures/abs-vs-emb}
    \caption{A visual comparison between abstractions (left) and embeddings (right). Note that abstractions (blue) have mappings to all variables in the high-level model, whereas the embeddings (orange) provides a fine-grained description only of the sub-system $X\rightarrow Y$.}
    \label{fig:visual-emb-vs-abs}
\end{figure}
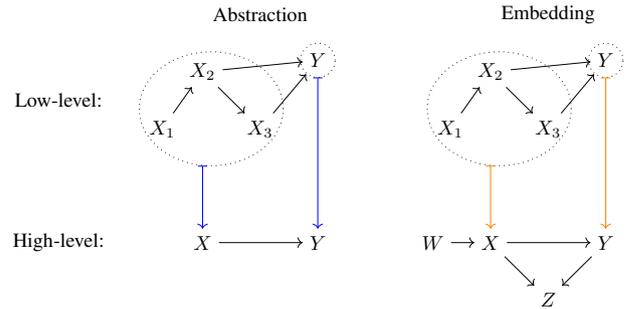

\paragraph{Contribution.}
In this paper, we extend the idea of abstractions: whereas abstractions describe how an entire high-level SCM can be described by a more detailed low-level SCM, we study \textit{causal embeddings}, describing how sub-systems of a high-level SCM can be described by detailed low-level models. We will (i) illustrate how this new point of view allows us to define a high-level causal model as the combination of multiple low-level sub-system models, (ii) discuss the graphical and functional consistency of causal embeddings, (iii) discuss theoretical applications of causal embeddings in a multi-level version of the marginal problem, and (iv) show the use of causal embeddings as a tool to merge overlapping datasets with differing levels of detail.

%% file: Figures/abs-vs-emb.tex
    \resizebox{\linewidth}{!}{
        \begin{tikzpicture}
            \node (M) at (-1.5,.5) {Low-level:};
            \node (abs) at (2,2) {Abstraction};
            \node (abs) at (7,2) {Embedding};
            \node (X1) at (0.3,0) {$X_1$};
            \node (X2) at (1,1) {$X_2$};
            \node (X3) at (2,0) {$X_3$};
            \node (Y) at (3,1.2) {$Y$};
    
            \draw[->] (X1) -- (X2);
            \draw[->] (X2) -- (X3);
            \draw[->] (X3) -- (Y);
            \draw[->] (X2) -- (Y);
    
            \node (M) at (-1.5,-2) {High-level:};
            \node (Xp) at (1,-2) {$X$};
            \node (Yp) at (3,-2) {$Y$};
    
            \draw[->] (Xp) -- (Yp);
    
            \draw[dotted] (1.15, .35) ellipse (1.25 and 1);
            \draw[dotted] (Y) ellipse (.3 and .3);
            \draw[|->, blue] (1, -.65) -- (Xp);
            \draw[|->, blue] (3,.9) -- (Yp);

            \node (X1emb) at (5.3,0) {$X_1$};
            \node (X2emb) at (6,1) {$X_2$};
            \node (X3emb) at (7,0) {$X_3$};
            \node (Yemb) at (8,1.2) {$Y$};
    
            \draw[->] (X1emb) -- (X2emb);
            \draw[->] (X2emb) -- (X3emb);
            \draw[->] (X3emb) -- (Yemb);
            \draw[->] (X2emb) -- (Yemb);
    
            \node (Xpemb) at (6,-2) {$X$};
            \node (Wpemb) at (5,-2) {$W$};
            \node (Ypemb) at (8,-2) {$Y$};
            \node (Zpemb) at (7,-3) {$Z$};
    
            \draw[->] (Xpemb) -- (Ypemb);
            \draw[->] (Wpemb) -- (Xpemb);
            \draw[->] (Xpemb) -- (Zpemb);
            \draw[->] (Ypemb) -- (Zpemb);
    
            \draw[dotted] (6.15, .35) ellipse (1.25 and 1);
            \draw[dotted] (Yemb) ellipse (.3 and .3);
            \draw[|->, orange] (6, -.65) -- (Xpemb);
            \draw[|->, orange] (8,.9) -- (Ypemb);
        \end{tikzpicture}
    }

%% file: Sections/1.5-Related-Works.tex
\subsection{Related works}
Causal abstractions can be described as SCM-to-SCM mappings; two frameworks have been proposed: the category theoretical $\boldalpha$ framework \citep{rischel2020category, rischel2021compabstraction} and the $\tau$-$\omega$ framework \citep{rubenstein2017causal, beckers2019abstracting,beckers2020approximate}. Alternatively, abstractions can be described as DAG-to-DAG mappings from one causal graph to another \citep{anand2023causal}. The relations between these abstraction frameworks have been studied by \citet{schooltink2025aligning}. The closest work to our proposal is \citet{otsuka2022equivalence}, who adopts a non-surjective definition of abstraction in their proposed $\phi$ framework. However, our definition is more flexible and extends to the case of embedding multiple low-level models into a single high-level model. 

Embeddings have relation to both the statistical \citep{kellerer1964masstheoretische} and causal marginal problem \citep{causal-marginal-problem}. Given multiple marginal SCMs \citet{causal-marginal-problem} proposes a method to generate a family of compatible joint SCMs, and through falsification find those that are counterfactually consistent. We show that embeddings can be used to tackle a multi-resolution version of the causal marginal problem. Relatedly, \citet{mejia2022obtaining} discusses how SCMs can be learned for overlapping statistical datasets under certain assumptions, while the Integration of Overlapping Datasets algorithm \citep{pmlr-v15-tillman11a} proposes a sound method to merge causal datasets and learn an equivalence class of graphs representing the data generation. This work was extended to a more special case by \citet{bang2025constraint} exploiting potential knowledge of variable ordering. 

Finally, embeddings can be used to merge causal datasets. Causal abstractions has already been used before to merge datasets in order to improve statistical power \citep{zennaro2023jointly,felekis2024causal}, although limited to transporting data from one model to another one. In the causal literature, \citet{janzing2018merging} shows how causal models can help in merging overlapping statistical datasets, while exploiting data from different models has been studied in the context of data fusion and transportability \citep{BareimboimPearl2016datafusion,pearl2022external}, including the case of multiple environments \citep{bareinboim2013meta}; these approaches, however, do not establish an explicit relation between SCMs, but rely on graphical calculus to take best advantage of observational and interventional data from a source and target model. 

%% file: Sections/2-Preliminaries.tex
\section{Preliminaries}
We first introduce the necessary background, specifically we present SCMs, Pearl's Causal Hierarchy, causal abstraction, and the causal marginal problem.

\paragraph{Notation.}
Throughout we will have a set of variables represented using bold uppercase \vars, and a specific variable using regular uppercase with index subscript when necessary $V_i\in\vars$. The value of a set of variables is indicated by bold lowercase $\mathbf{v}$, and the value of a single variable using regular lowercase $v_i$, again with index subscript when necessary. Additionally, a distribution over a variable $V$ or set of variables $\vars$ is denoted as $P(V)$ and $P(\vars)$, respectively.

\subsection{Causality}
\paragraph{Structural Causal Models.}
SCMs are formal descriptions of causal models, specifying causal variables and effects among them \citep{pearl2000models}. We define an SCM as follows:

\begin{definition}[Structural Causal Model]\label{def:SCM}
    An SCM is a 4-tuple $\model:\langle\exos,\vars,\funcs,\dists\rangle$, with:
    \begin{itemize}
        \item $\exos$: a set of unobservable (exogenous) variables, which can take values in the range $\range(\exos)$,
        \item $\vars$: a set of observable (endogenous) variables, which can take values in the range $\range(\vars)$,
        \item $\funcs$: a collection of functions determining the value of the endogenous variables $\vars$, such that for each $V\in\vars$ there exists a function $f_V(Pa_V,\exos_V)$ with $Pa_V\subseteq \vars \setminus V \text{ and }\exos_V\subseteq \exos,$
        \item $\dists$: a probability distribution over the exogenous variables $\exos$.
    \end{itemize}
\end{definition}

An SCM $\model$ entails a unique directed graph $G:\langle\vars,\mathbf{E}\rangle$ with the vertices given by the endogenous variables $\vars$ and an edge $V_i\rightarrow V_j \in \mathbf{E}$ if the value of $V_j$ depends on $V_i$. Specifically, $V_i\rightarrow V_j \in \mathbf{E}$ if $V_i \in Pa_{V_j}$. Thus the set $Pa_{V_j}$ is the set of parents of $V_j$ in a graph theoretical sense. Additionally, whenever two variables $V_i, V_j$ share an unobservable parent (confounder) $U\in\exos$ we add a dashed bidirected arrow $V_i\confarrow V_j$. We assume the SCMs to have no cyclic relations: the graph entailed by the SCM is acyclic (DAG). Notice that the bidirected arrow is a shorthand for $V_i\leftarrow U \rightarrow V_j$; hence the graph $G$ remains acyclic.

Causal models allow for reasoning about causality through interventions. In this paper we consider \emph{hard interventions}:

\begin{definition}[Intervention]\label{def:intervention}
Given a causal model $\model:= \langle\vars,\exos,\funcs,\dists\rangle$ an intervention on a variable $X\in\vars$, denoted as $do(X = x)$, is the replacement of the function $f_{X}\in\funcs$ with a constant function $f'_{X} = x$.
\end{definition}

As is common, we will apply the $do$-operator over sets of variables $\mathbf{X}\subseteq\vars$, implying the replacement of the functions $f_X\in\funcs$ for all $X\in\mathbf{X}$. In essence, when intervening on a variable $X$ its value is no longer dependent on its parents but forcibly set to some value $x$. Consequently, as the DAG implied by an SCM has edges determined by the functions $\funcs$, interventions have graphical implications: specifically, an intervention $do(X)$ removes all edges coming into $X$ in the DAG.

\subsection{Pearl's Causal Hierarchy.}
In causal reasoning there are three distinct types of questions one may wish to answer: (i) given the observation $X$ what can be said about $Y$, (ii) if $X$ is set to $X=x$ what can be said about $Y$, (iii) given  observations $Y=y$ and $X=x$ what can be said about $Y$ if $X$ had been set to $X=\hat{x}$? These questions can be categorized as (i) observing $\lone$, (ii) acting $\ltwo$, and (iii) imagining $\mathcal{L}_3$. This defines the Pearl's Causal Hierarchy (PCH) \citep{pearl2000models,Bareinboim2020hierarchy}, where each layer subsumes the previous layer, but cannot be reduced to the previous one: a query at layer $\li$ can be answered at layer $\mathcal{L}_{i+1}$, but a query in layer $\mathcal{L}_{i+1}$ cannot in general be reduced to layer $\li$. Fully specified SCMs allow for reasoning about counterfactuals $\mathcal{L}_3$.  

\paragraph{Notation.} For generalization, we will use the $\li$ operator in distributions: instead of $P(Y\vert X)$ we write $P(Y\vert \lone(X))$, and instead of $P(Y\vert do(X))$ we write $P(Y\vert \ltwo(X))$. We will not consider $\mathcal{L}_3$ quantities in this work. 

\subsection{Causal Abstractions}
Causal abstractions provide tools that allow us to map a detailed low-level causal model to a coarser high-level model. Similar to how causal models have a functional and graphical side (SCMs and DAGs, respectively), abstractions can be defined on both the functional and graphical level.

\paragraph{Functional.} 
We first consider the functional side through the framework of the \alphabs{} \citep{rischel2020category}.

\begin{definition}[$\alpha$-abstraction]\label{def:alpha-abs}
    Let $\model: \langle \vars_{\model},\exos_{\model},\funcs_{\model},$ $P(\exos_{\model})\rangle$ and $\mathcal M':\langle \vars_{\model'},\exos_{\model'},\funcs_{\model'},P(\exos_{\model'})\rangle$ be two SCMs, then an \alphabs{} $\boldsymbol{\alpha}: \mathcal M \rightarrow \mathcal M'$ is given by a 3-tuple $\langle \mathbf R, \varphi, \alpha_{V'} \rangle$ with:
    \begin{enumerate}
        \item $\mathbf R\subseteq \mathbf V_\mathcal M$ is a subset of relevant variables in $\model$.
        \item $\varphi: \mathbf R \rightarrow \mathbf V_{\model'}$ is a surjective map from the relevant variables to the variables of $\model'$.
        \item $\alpha_{V'}: \range(\varphi^{-1}(V')) \rightarrow \range(V')$, for each $V' \in \mathbf{V}_{\model'}$, is a surjective function from the range of the variables in the pre-image $\varphi^{-1}(V') \subseteq \vars_\model$ in $\mathcal M$ to the range of the variable $V'$ in $\mathcal M'$.
    \end{enumerate}
\end{definition}

Importantly, the application of $\alpha_\vars$ to a distribution is given by the pushforward $\alpha_\vars\left[P(\vars)\right] = {\alpha_\vars}_{\#}(P)(\vars)$.

\alphabs{}s do not enforce any form of causal consistency between a base model and an abstracted model; any mapping, as long as surjective, is permissible. For this reason, an error measure is used to define how much an abstracted model agrees with a base model. For \alphabs s the error measure is defined as follows: 

\begin{definition}[$\li$-Abstraction error]\label{def:func-consistency}
    Let $\boldalpha:\model\rightarrow\model'$ be an $\boldalpha$-abstraction and $\mathbf{X}',\mathbf{Y}' \subseteq \vars_{\model'}$, its $\li$-error is given by the distance or divergence $D$ between the distribution obtained by first abstracting and then evaluating: $${\color{blue}P_{\model'}(\mathbf Y'\:|\:\alpha_{\mathbf{X}'}[\li(\varphi^{-1}(\mathbf{X}'))])},$$ and that obtained by first evaluating and then abstracting: $${\color{orange}\alpha_{\mathbf{Y}'}\left[P_{\model}(\varphi^{-1}(\mathbf Y')\:|\: \li(\varphi^{-1}(\mathbf X')))\right]},$$ and taking the maximum over all $\mathbf{X}', \mathbf{Y}' \subseteq \vars_{\model'}$ as the error. Visually this corresponds to the maximum distance between the paths in blue and orange over all diagrams of the following form:
    \begin{center}
        \begin{tikzpicture}
            \node (Y) at (0,0) {$\varphi^{-1}\left(\mathbf{Y}'\right)$};
            \node (YX) at (5,0) {$\varphi^{-1}\left(\mathbf{Y'}\right)\:\vert\: \li\left(\varphi^{-1}(\mathbf{X}')\right)$};
            \node (Yp) at (0,-1.5) {$\mathbf{Y}'$};
            \node (YpXp) at (5,-1.5) {$\mathbf{Y}'\:\vert\: \li(\mathbf{X}')$};
    
            \draw[->, blue] (Y) -- (YX) node [midway, above] {\small$\li(\varphi^{-1}(\mathbf X'))$};
            \draw[->, orange] (Yp) -- (YpXp) node [midway, below] {\small$\li(X')$};
            \draw[->, orange] (Y) -- (Yp) node [midway, left] {\small$\alpha_{\mathbf{Y}'}$};
            \draw[->, blue] (YX) -- (YpXp) node [midway, right] {\small$\alpha_{\mathbf{X}'}$};
        \end{tikzpicture}
    \end{center}
    or formally by the following equation:
    \begin{multline}
            \hspace{-4pt}e_{\li}(\boldsymbol{\alpha})=
            \max_{\mathbf X', \mathbf Y'\subseteq \vars_{\model'}} D
            \left({\color{blue}P_{\model'}(\mathbf Y'|\alpha_{\mathbf{X}'}[\li(\varphi^{-1}(\mathbf{X}'))])},\right. \\ 
            \left.{\color{orange}\alpha_{\mathbf{Y}'}\left[P_{\model}(\varphi^{-1}(\mathbf Y')\:|\: \li(\varphi^{-1}(\mathbf X')))\right]}\right)
    \end{multline}
\end{definition}

An abstraction is \emph{$\li$-consistent} if its $\li$-error is zero.

\paragraph{Graphical.}
On the graphical level abstractions are defined as relations between the graphs of causal models. We will consider the framework of the Cluster DAG (CDAG), as introduced by \citet{anand2023causal}.

\begin{definition}[Cluster DAG]
    Let $\model$ and $\model'$ be two causal models admitting DAGs $G_{\model}:\langle\mathbf V_{\model}, \mathbf E_{\model}\rangle$ and $G_{\model'}:\langle \mathbf V_{\model'}, \mathbf E_{\model'}\rangle$, respectively; let $\varphi:\vars_{\model}\rightarrow\vars_{\model'}$ be a surjective map from the variables of $\model$ to the variables of $\model'$. 
    $G_{\model'}$ is a Cluster DAG of $G_{\model}$ if:
    \begin{enumerate}
        \item A directed edge $V_i' \rightarrow V_j'$ is in $E_{\model'}$ iff there exists an edge $V_n \rightarrow V_m \in E_{\model}$ such that $\varphi(V_n) = V_i'$ and $\varphi(V_m) = V_j'$.
        \item A bidirected edge $V_i' \confarrow V_j'$ is in $E_{\model'}$ iff there exists an edge $V_n \confarrow V_m \in E_{\model}$ such that $\varphi(V_n) = V_i'$ and $\varphi(V_m) = V_j'$.
    \end{enumerate}
\end{definition}

While graphical models do not specify distributions and functions explicitly they do entail constraints on distributional (in)equalities. For example an edge $X\rightarrow Y$ implies $P(Y\vert X)\neq P(Y)$. We denote $\mathcal{G}^{\li}(G)$ the set of algebraic constraints on $\li$ distributions implied by the causal graph $G$. We can then define a notion of consistency for graphical abstractions similar to that in functional abstractions.

\begin{definition}[Graphical $\li$-Consistency]\label{def:graph-consistency}
    Let $G_{\model}$ be the causal graph induced by SCM $\model$, and $G_{\model'}$ an abstraction of $G_{\model}$ with the variable map $\varphi:\vars_{\model}\rightarrow\vars_{\model'}$. Let us define $\mathcal{G}^{\li}(G_{\model'}^{-1})$ the set of all the constraints obtained from $\mathcal{G}^{\li}(G_{\model'})$ substituting each variable $V'\in\vars_{\model'}$ with the respective pre-image $\varphi^{-1}(V')$. Then $G_{\model'}$ is graphically \li-consistent with $G_{\model}$ iff
    \begin{equation}
        \mathcal{G}^{\li}\left(G_{\model'}^{-1}\right) \subseteq \mathcal{G}^{\li}\left(G^{ }_{\model^{ }}\right). 
    \end{equation}
\end{definition}

Graphical consistency and functional consistency align as follows: a graphically $\li$-consistent abstraction from $G_{\model}$ to $G_{\model'}$ implies the existence of an SCM $\model'$ that is functionally $\li$-consistent with a given $\model$; a functionally $\li$-consistent abstraction implies graphical $\li$-consistency only if all causal dependencies are preserved (see for more details \citet{schooltink2025aligning}).

\subsection{(Causal) Marginal Problem}
We will show applications of our work in the marginal problem: the challenge of finding a joint probability given two or more separate but overlapping datasets \citep{kellerer1964masstheoretische}, for example estimating the joint distribution $\hat{P}(X,Y,Z)$ from the two distributions $\hat{P}(X,Y)$ and $\hat{P}(Y,Z)$.

\begin{definition}[Marginal Problem]\label{def:marginal-problem}
    Given \emph{marginal} datasets $\mathcal{X}_{\vars_1},\dots,\mathcal{X}_{\vars_n}$, defined respectively over variables $\vars_i$, with possibly non-empty intersections of variables between datasets: $\vars_i \cap\vars_j \neq \emptyset$, find the joint distribution over the union $\bigcup_{i=1}^{n} \vars_i$.
\end{definition}

Overlap between the variables of the different datasets is a necessary condition in order to find a meaningful solution to the marginal problem. To illustrate this, consider the counter example where we have two datasets that do not overlap at all; this gives no information on dependencies and as such on the joint distribution of their variables.

The marginal problem per Def.\ref{def:marginal-problem} is a statistical problem: a question of combining observational distributions. We will focus on the causal extension of the marginal problem \citep{causal-marginal-problem}: finding a joint SCM from multiple overlapping marginal SCMs. Whereas in the statistical setting the object of interest is distributions, in the causal setting it is SCMs. This distinction is required as observational distributions alone cannot fully describe causality.

\begin{definition}[Causal Marginal Problem]\label{def:causal-marginal-problem}
    Given SCMs $\model_1, \dots, \model_n$, with possibly non-empty intersections of endogenous variables between models $\vars_{\model_i} \cap\vars_{\model_j} \neq \emptyset$, find the space of joint causal models $\model^*$ over the union $\bigcup_{i=1}^n \vars_{\model_i}$ consistent with the models $\model_1, \dots, \model_n$.
\end{definition}

Our proposed causal embeddings will tackle the more complicated setting of the causal marginal problem where the overlapping variables do not share the same level of detail.

%% file: Sections/3-Embeddings.tex
\section{SCM Projections}
In this section we will formalize a notion of projections for SCMs which will be instrumental in defining embeddings. First, recall that SCMs consist of a set of observable variables $\vars$ and a set of unobservable variables $\exos$. One may want to move some variables $V\subset\vars$ to the set of exogenous variables $\exos$, perhaps to simplify the model or since some variables have become unobservable.
Formally we can describe a projection of SCMs as follows:
\begin{definition}[SCM Projection]\label{def:SCM-projection}
    Let $G_{\model}$ be the graph induced by an SCM $\model$, $\relevant\subseteq \vars_{\model}$ a set of relevant variables to be preserved, $\model'$ an SCM over the relevant variables s.t. $\vars_{\model'} = \relevant$, and $G_{\model'}$ the graph induced by $\model'$. $G_{\model'}$ is a projection of $G_{\model}$ iff $G_{\model'}$ is graphically $\ltwo$-consistent with $G_{\model}$ with $\varphi$ the identity mapping.
\end{definition}

Importantly, SCM projections enforce only graphical constraints. The SCM projection $\model'$ of an SCM $\model$ implies the compatibility of causal (in)dependencies; that is, if there is a causal arrow $X\rightarrow Y$ in $\model$ and both $X$ and $Y$ are preserved, then there must exist an arrow $X\rightarrow Y$ in the graph induced by $\model'$. 
This definition follows as an extension of latent structure projections \citep{pearl1995projection} (see App.\ref{app:latent}) to the domain of SCMs. 

\section{Embeddings}
We will now define embeddings as a generalization of abstractions and similarly provide the notions of functional and graphical consistency. First, recall that an \boldalpha-abstraction describes the relation of a causal model to a coarser causal model. \boldalpha-abstractions enforce surjectivity of both the mapping of variables $\varphi$ and the mapping $\alpha_{\vars}$ of their ranges. The surjectivity requirement makes sure that all states of the abstracted model are represented in the base model. So a base model cannot represent only a sub-system of the coarse model. Dropping the requirement of surjectivity of $\varphi$ allows for defining refinements of sub-systems of a larger system. To give more intuition to this idea, consider the following motivating example: 

\begin{example}[Simplified Ecosystem Modeling]\label{ex:ecosystem}
    Imagine we model the causal dynamics of some ecosystem of deer and squirrels, and the effects of predators and human hunting on their population, as in the causal model shown in Fig.\ref{fig:coarse-model}:
    
    \begin{figure}
        \centering
        \input{Figures/coarse-model}
        \caption{A high-level causal model of a simplified ecosystem.}
        \label{fig:coarse-model}
    \end{figure}
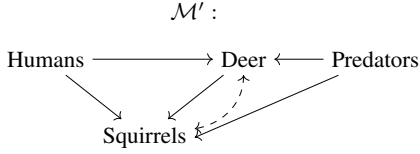

    However, for the area we are interested in, such causal models have not been defined. Instead other researchers have modeled two related systems as shown in Fig.\ref{fig:fig-detailed-models}: model $\model_1$, describing causal interactions between human hunting, squirrels, deer, and berry bush availability; and model $\model_2$, describing the causal interactions between wolves, eagles, red deer, fallow deer, and squirrels.

    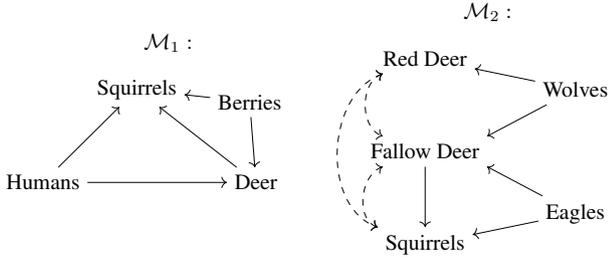
\begin{figure}
        \centering
        \input{Figures/detailed-models}
        \caption{Two low-level causal models, each modeling a sub-system of a simplified ecosystem.}
        \label{fig:fig-detailed-models}
    \end{figure}
    
    Notice that $\model_1$ and $\model_2$ together cover all variables we are interested in, albeit both $\model_1$ and $\model_2$ contain more detail. To construct a single model for the whole system, both models should be mapped to the same level of detail.
\end{example}

Following Ex.\ref{ex:ecosystem}, we formalize causal embeddings.

\paragraph{Formal Definition.}
We define an embedding using a generalization of the \alphabs{} framework by first expanding the \alphabs{} definition to allow for non-surjective maps $\varphi$ wrt $\vars_{\model'}$:

\begin{definition}[Non-surjective \alphabs{}]\label{def:Non-Sur-alpha-abs} 
    A non-surjective \alphabs{} is an \alphabs{} with:    
    \begin{enumerate}
        \item $\relevant \subseteq \vars_\model$ is a subset of relevant variables in $\model$.
        \item $\relevanttarget \subseteq \vars_{\model'}$ is a subset of relevant variables in $\model'$.
        \item $\varphi: \relevant \rightarrow \relevanttarget$ is a surjective map between relevant variables.
        \item $\alpha_{V'}: \range(\varphi^{-1}(V')) \rightarrow \range(V')$, for each $
        V' \in \relevanttarget$, is a surjective function from the range of the pre-image $\varphi^{-1}(V') \subseteq \relevant$ in $\model$ to the range of $V'$ in $\model'$.
    \end{enumerate}
\end{definition}

Similar to \alphabs{}, Def.\ref{def:Non-Sur-alpha-abs} does not enforce consistency or make any graphical guarantees. However, since embeddings are to encode a detailed description of a sub-system into a high-level model, it is important to ensure a compatible embedding. For this purpose we define embeddings as follows, by including graphical constraints: 

\begin{definition}[\boldalpha-embedding]\label{def:alpha-emb} 
    Given SCMs $\model$ and $\model'$ and a non-surjective \alphabs{} $\boldalpha$ with $\varphi:\relevant\rightarrow\relevanttarget$, \boldalpha{} is an \boldalpha-embedding iff the projection of the graph $G_{\model'}$ over $\relevanttarget$ is a CDAG of the projection of the graph $G_{\model}$ over $\relevant$.
\end{definition}

Let us illustrate \boldalpha-embeddings on our guiding example.

\begin{example}(Simplified Ecosystem Modeling)\label{ex:ecosystem2}
    Let us consider again the ecosystem models in Ex.\ref{ex:ecosystem} and suppose we define an \boldalpha-embedding $\alpha_1: \model_1\rightarrow\model'$ with:
    \begin{gather*}
    \varphi_{1}:\relevant_{\model_1} \rightarrow \relevanttarget_{\model'}:=
    \begin{cases}
        \text{Humans} &\mapsto \text{Humans}\\
        \text{Squirrels} &\mapsto \text{Squirrels}\\
        \text{Deer} &\mapsto \text{Deer}\\
    \end{cases}
    \end{gather*}

    $\alpha_1$ satisfies the \boldalpha-embedding definition as the projection of $G_{\model'}$ over $\relevanttarget_{\model'} = \{ Humans, Squirrels, Deer\}$ is a CDAG of the projection of $G_{\model_{1}}$ over $\relevant_{\model_1} = \{ Humans, Squirrels, Deer\}$.
    
\end{example}

An alternative definition makes a more explicit reference to the graphical constraints allowing for easier construction and verification of embeddings.
For this, we need the definition of mediated adjacencies and confounders \citep{massidda2024learning,schooltink2025aligning}.

\begin{definition}[Mediated adjacencies]
    Given a set of relevant variables $\relevant \subseteq \vars$ and $X,Y\in\relevant$, a mediated adjacency $X \rightsquigarrow Y$ wrt $\relevant$ is a directed path from $X$ to $Y$ such that all intermediate variables are not in $\relevant$.
\end{definition}

\begin{definition}[Mediated confounders]
    Given a set of relevant variables $\relevant \subseteq \vars$ and $X,Y\in\relevant$, a mediated confounder $X\medconfarrow Y$ wrt $\relevant$ is a hidden or observed confounder $Z \notin \relevant$, with $Z\rightsquigarrow X$ and $Z\rightsquigarrow Y$.
\end{definition}

We can now give a constructive definition of \boldalpha-embedding.

\begin{definition}[\boldalpha-embedding (Alternative)]\label{def:alpha-emb-alt}
    Given causal models $\model$ and $\model'$ and a non-surjective \alphabs{} with $\varphi:\relevant\rightarrow\relevanttarget$, \boldalpha{} is an \boldalpha-embedding iff the following two conditions hold:
    
    \begin{enumerate}
        \item a mediated adjacency $X'\rightsquigarrow Y'$ wrt $\relevanttarget$ is in $\model'$ iff there exists $X \in \varphi^{-1}(X'), Y \in \varphi^{-1}(Y')$ such that there is a mediated adjacency  $X\rightsquigarrow Y$ wrt $\relevant$.
        \item a mediated confounder $X'\medconfarrow Y'$ wrt $\relevanttarget$ is in $\model'$ iff there exists $X \in \varphi^{-1}(X'), Y \in \varphi^{-1}(Y')$ such that there is a mediated confounder $X\medconfarrow Y$ wrt $\relevant$.
    \end{enumerate}
\end{definition}

We can prove the equivalence of the two definitions (see proof in App.\ref{app:proof-eq-defs}):

\begin{lemma}[Equivalence Def.\ref{def:alpha-emb} and Def.\ref{def:alpha-emb-alt}]\label{lem:eq-defs}
    Def.\ref{def:alpha-emb} in terms of projections is equivalent to Def.\ref{def:alpha-emb-alt} in terms of explicit graphical constraints. 
\end{lemma}

Note that the definition of \boldalpha-embedding only restricts the edges that are part of mediated adjacencies and confounders, any edge in the graph of $\model'$ that is not part of either of these may exist (see following example). 

\paragraph{Example.}(Example of an Embedding)\label{ex:emb}
\textit{In this example we illustrate an example of a permissible embedding, and highlight some graphical properties. The following diagram illustrates an $\boldalpha$-embedding, with non-relevant variables denoted as {\color{gray}$\bullet$}:}
    \begin{center}
        \input{Figures/embedding-example}
    \end{center}
    \textit{The diagram shows:
    \begin{enumerate}
        \item[(i)] How paths in the low-level model (left) can be mapped to one or multiple paths in the high-level model. The path from the cluster $\{X_1,X_2\}$ to $\{Y\}$ is represented in the high-level model with two paths from $X'$ to $Y'$, one containing two mediating variables. Embeddings preserve causal effects, but do not enforce causal pathways beyond consistency.
        \item[(ii)] How confounders are mapped to confounders. The unobservable confounder in the low-level model between $\{Y\}$ and $\{W\}$ is represented in the high-level model with an observable confounder between $Y'$ and $W'$. Embeddings require confounding effects to be consistent, but unobservable confounders may become observable or vice-versa.
        \item[(iii)] How additional structures may exist in the high-level model. The high-level graph has an additional parent for $X'$, child of $Y'$ and collider structure with parents $X', W', {\color{gray}\bullet}$. Any number of variables can exist that do not create new causal pathways between the relevant variables.
        \item [(iv)] How some embedding might be unintuitive. The high-level model may introduce an observable separation $X' \rightarrow {\color{gray}\bullet} \rightarrow W'$ which is not represented in the low-level model $X_2 \rightarrow W$. 
    \end{enumerate}}

As illustrated, \boldalpha-embedding are versatile. If patterns as (iv) in the example are deemed undesirable, one can adopt additional restrictions; for example, requiring a graph homomorphism from the graph of $\model$ to a subgraph of $\model'$ respecting the variable mapping, as is in $\phi$-abstractions \citep{otsuka2022equivalence}. For the sake of generality, we do \emph{not} impose such a restriction.

Reconsidering the motivating example, \boldalpha-embeddings provide a formal tool to specify all the needed embeddings: 

\begin{example}(Simplified Ecosystem Modeling)\label{ex:finished-maps}
    Let us consider again Ex.\ref{ex:ecosystem2} and let us define also the last \boldalpha-embedding for model $\model_2$ onto $\model'$.
\begin{gather*}
\varphi_2:\vars_{\model_2} \rightarrow \vars_{\model'}:=
\begin{cases}
    \text{Eagles} &\mapsto \text{Predators}\\
    \text{Wolves} &\mapsto \text{Predators}\\
    \text{Red Deer} &\mapsto \text{Deer}\\
    \text{Fallow Deer} &\mapsto \text{Deer}\\
    \text{Squirrels} &\mapsto \text{Squirrels}\\
\end{cases}
\end{gather*}
This completes the example with a graphical specification of the \boldalpha-embeddings necessary to map $\model_1, \model_2$ onto $\model'$.
\end{example}

\subsection{Causal Consistency of Embeddings} 
Similar to abstractions, it is important to assess the consistency of causal embeddings. Therefore, we will define a notion of consistency for embeddings, both in terms of a functional consistency, wrt the functional descriptions of SCMs, and graphical consistency, wrt to the causal graphs.

\paragraph{Functional Consistency.} First let us consider the functional side of the \boldalpha-embedding: the non-surjective \alphabs{}. We adapt the $\li$-abstraction error from Def.\ref{def:func-consistency}:

    \begin{definition}[$\li$-Embedding error]\label{def:emb-consistency}
    Given an \boldalpha-embedding $\boldalpha:\model\rightarrow\model'$ and $\mathbf{X}',\mathbf{Y}' \subseteq \relevanttarget$, the $\li$-error is given by the maximum distance or divergence between the distribution obtained by first embedding and then evaluating and the distribution obtained by first evaluating and then embedding, as given by the formula:
    \begin{multline}
            \hfill e_{\li}(\boldsymbol{\alpha})=\hfill
            \max_{\mathbf X', \mathbf Y'\subseteq \relevanttarget} D
            \left(P_{\model'}(\mathbf Y'\:|\:\alpha_{\mathbf{X}'}[\li(\varphi^{-1}(\mathbf{X}'))]),\right. \\ 
            \left.\alpha_{\mathbf{Y}'}\left[P_{\model}(\varphi^{-1}(\mathbf Y')\:|\: \li(\varphi^{-1}(\mathbf X')))\right]\right)
    \end{multline}
\end{definition}

An embedding with $\li$-error zero is $\li$-consistent. Functional consistency does not follow by construction, as this may be too limiting and some error may be tolerable, similar to abstractions. Notice that the embedding error subsumes the abstraction error and can be considered a generalization, mirroring the definition of the non-surjective \alphabs.

\begin{remark}\label{rem:1}
    Let $\mathbf{M}:=\{\model_1,\dots,\model_n\}$ be a set of $n$ causal models and $\mathbf{A}:=\{\boldalpha_1,\dots,\boldalpha_n\}$ be a set of $n$ \boldalpha-embeddings such that $\boldalpha_i:\model_i\rightarrow\model'$ embeds into the same $\model'$ for all $i$. \li-consistency of all $\boldalpha\in\mathbf{A}$ solely does \emph{not} imply uniqueness of the \li-distributions over the set of relevant variables $\relevanttarget$ for any permissible $\model'$, see App.\ref{app:rem-1} for a concrete example.
    Instead, to guarantee uniqueness of all distributions stronger assumptions on the set of embeddings $\mathbf{A}$ are needed. We conjecture a sufficient condition would be that all causal and confounding edges are constrained by at least one of the embeddings $\boldalpha\in\mathbf A$.
\end{remark}

\paragraph{Graphical Consistency.}
Additional to the functional side, embeddings have a graphical restriction, so we will also consider graphical consistency by adjusting for unmapped variables in the high-level model. For this we introduce the following notation:

\begin{definition}[Graphical $\li$-Embedding Consistency] \label{def:graph-emb-consistency}
Given a variable map $\varphi:\relevant\rightarrow\relevanttarget$, let $G_{[\relevant]}$ be a projection of the DAG of $\model$ over $\relevant$ and $G_{[\relevanttarget]}$ a projection of the DAG of $\model'$ over $\relevanttarget$.
The two models $\model$ and $\model'$ are graphically consistent iff 
\begin{equation}
    \mathcal{G}^{\li}\left(G_{\left[\relevanttarget\right]}^{-1}\right) \subseteq \mathcal{G}^{\li}\left(G^{ }_{\left[\relevant\right]}\right). 
\end{equation}
\end{definition}

As the \boldalpha-embedding is defined in terms of projections and CDAGs it inherits graphical $\ltwo$-consistency (see proof from Lem.\ref{lem:eq-defs}: App.\ref{app:proof-eq-defs}):

\begin{theorem}[\boldalpha-embedding is graphically \ltwo-consistent]\label{thm:emd-l2}
    By definition an \boldalpha-embedding is graphically \ltwo-consistent.
\end{theorem}

A relation between functional and graphical consistency for embeddings can be proved in analogy to the same relation proved for abstractions (proofs in App.\ref{app:graph-implies-func} and App.\ref{app:func-not-implies-graph-proof}):

\begin{theorem}[Graphical $\ltwo$-consistent map $\Rightarrow$ Functional $\ltwo$-consistent embedding]\label{thm:graph-implies-func}
    Given a graphically \ltwo-consistent $\varphi:\relevant\rightarrow\relevanttarget$ with $\relevant\subseteq\vars_{\model}$ and $\relevanttarget\subseteq\vars_{\model'}$ there exists a specification of $\model'$ and an \boldalpha-embedding $\boldalpha:\model\rightarrow\model'$ with variable map $\varphi$ s.t. $\boldalpha$ is functionally \ltwo-consistent.
\end{theorem}

\begin{theorem}[Functional $\ltwo$-consistent non-surjective abstraction $\not\Rightarrow$ embedding]\label{thm:func-not-implies-graph}
    A non-surjective \ltwo-consistent \boldalpha-abstraction does not necessarily imply an \boldalpha-embedding.
\end{theorem}

Notice that the asymmetry between Thm.\ref{thm:graph-implies-func} and Thm.\ref{thm:func-not-implies-graph} follows from Thm.\ref{thm:emd-l2}, as Def.\ref{def:alpha-emb} requires graphical \ltwo-consistency but not functional \ltwo-consistency.

%% file: Figures/coarse-model.tex
\begin{center}
    \normalfont
    \resizebox{.7\linewidth}{!}{
    \begin{tikzpicture}[x=-1cm]
        \node (model) at (.75,.75) {$\model':$};
        \node (Squirrels) at (1.5,-1.2) {Squirrels};
        \node (Deer) at (0,0) {Deer};
        \node (Humans) at (3,0) {Humans};
        \node (Predators) at (-2,-0) {Predators};

        \draw[->] (Humans) -> (Deer);
        \draw[->] (Humans) -> (Squirrels);
        \draw[->] (Predators) -> (Deer);
        \draw[->] (Predators) -> (Squirrels.east);
        \draw[->] (Deer) -> (Squirrels);
        \draw[<->, dashed] (Deer) to[out=270, in=10] (Squirrels);
    \end{tikzpicture}
    }
\end{center}

%% file: Figures/detailed-models.tex
\begin{center}
        \normalfont
        \resizebox{\linewidth}{!}{
        \begin{tikzpicture}[x=-1cm]
            \node (model2) at (-2.1,-1.75) {$\model_2:$};
            \node (Red Deer) at (-1.1,-2.5) {Red Deer};
            \node (Fallow Deer) at (-1.1,-4) {Fallow Deer};
            \node (Rabbits3) at (-1.1,-5.5) {Squirrels};
            \node (Wolves) at (-3.5,-3) {Wolves};
            \node (Eagles) at (-3.5,-5) {Eagles};
    
            \draw[->] (Wolves) -> (Red Deer);
            \draw[->] (Wolves) -> (Fallow Deer.north east);
            \draw[->] (Eagles) -> (Fallow Deer.south east);
            \draw[->] (Eagles) -> (Rabbits3);
            \draw[->] (Fallow Deer) -> (Rabbits3);
            \draw[<->, dashed] (Red Deer) to[in=160, out=-160] (Fallow Deer);
            \draw[<->, dashed] (Fallow Deer) to[in=160, out=-160] (Rabbits3);
            \draw[<->, dashed] (Red Deer) to[in=160, out=-160] (Rabbits3);
    
            \node (model1) at (3,-2.25) {$\model_1:$};
            \node (Rabbits2) at (3.5,-3) {Squirrels};
            \node (Deer2) at (1.6,-4.5) {Deer};
            \node (Humans2) at (5,-4.5) {Humans};
            \node (Berries) at (1.7,-3.2) {Berries};
    
            \draw[->] (Humans2) -> (Deer2);
            \draw[->] (Deer2) -> (Rabbits2);
            \draw[->] (Berries) -> (Deer2);
            \draw[->] (Berries) -> (Rabbits2);
            \draw[->] (Humans2) -> (Rabbits2);
        \end{tikzpicture}
        }
    \end{center}

%% file: Figures/embedding-example.tex
\begin{tikzpicture}
    \node (X1) at (-.8,-.1) {$X_1$};
    \node (X2) at (.8,.1) {$X_2$};
    \node (Y) at (-.3,-2) {$Y$};
    \node (W) at (.8,-1.3) {$W$};
    \node (Z) at (2,-1.7) {\color{gray}$\bullet$};

    \draw[->] (X1) -- (X2);
    \draw[->] (X2) -- (W);
    \draw[->] (W) -- (Z);
    \draw[->] (X1) -- (Y);
    \draw[<->, dashed] (Y)  to[out=0, in=250] (W);

    \draw[dotted, rotate=7] (0, 0) ellipse (1.2 and .4);
    \draw[dotted] (-.3,-2) ellipse (.3 and .3); 
    \draw[dotted] (.8,-1.3) ellipse (.3 and .3);

    
    \node (Xp) at (5.2,0) {$X'$};
    \node (Yp) at (5.7,-2) {$Y'$};
    \node (Wp) at (6.8,-1.3) {$W'$};
    \node[minimum size=0pt,inner sep=1pt, outer sep=0pt] (Zp1) at (5,1) {\color{gray}$\bullet$};
    \node[minimum size=0pt,inner sep=1pt, outer sep=0pt] (Zp2) at (6,-.65) {\color{gray}$\bullet$};
    \node[minimum size=0pt,inner sep=1pt, outer sep=0pt] (Zp3) at (4.6,-.65) {\color{gray}$\bullet$};
    \node[minimum size=0pt,inner sep=1pt, outer sep=0pt] (Zp4) at (4.8,-1.4) {\color{gray}$\bullet$};
    \node[minimum size=0pt,inner sep=1pt, outer sep=0pt] (Zp5) at (6.6,-2.1) {\color{gray}$\bullet$};
    \node[minimum size=0pt,inner sep=1pt, outer sep=0pt] (Zp6) at (5.5,-3) {\color{gray}$\bullet$};
    \node[minimum size=0pt,inner sep=2pt, outer sep=0pt] (Zp7) at (6.6,0.1) {\color{gray}$\bullet$};

    \draw[->] (Zp1) -- (Xp);
    \draw[->] (Xp) -- (Zp2);
    \draw[->] (Zp2) -- (Wp);
    \draw[->] (Xp) -- (Yp);
    \draw[->] (Xp) -- (Zp3);
    \draw[->] (Zp3) -- (Zp4);
    \draw[->] (Zp4) -- (Yp);
    \draw[->] (Zp5) -- (Yp);
    \draw[->] (Zp5) -- (Wp);
    \draw[->] (Yp) -- (Zp6);
    \draw[->] (Xp) -- (Zp7);
    \draw[->] (Wp) -- (Zp7);
    \draw[->] (Zp2) -- (Zp7);

    \node (emb1) at (3,0) {\scriptsize $\{X_1,X_2\}\mapsto X'$};
    \node (emb2) at (3,-.4) {\scriptsize $\{Y\}\mapsto Y'$};
    \node (emb3) at (3,-.8) {\scriptsize $\{W\}\mapsto W'$};
    \draw[->, double] (2,-1) -- (4,-1);
    
\end{tikzpicture}

%% file: Sections/4-Marginal-Problem.tex
\section{Multi-Resolution Marginal Problem}
We now show how $\boldalpha$-embeddings can be used to expand upon the causal marginal problem. 
Recall the causal marginal problem (Def.\ref{def:causal-marginal-problem}) where marginal models have sets of overlapping variables. We consider the case when this assumption does not hold: that is, marginal models have different representations of the overlapping variables, either by (i) having differing resolutions, or (ii) having the overlapping variables represented using multiple variables. We refer the reader again to Ex.\ref{ex:ecosystem}-\ref{ex:finished-maps}, where the detailed models have overlapping variables through the embedding.

We define the Multi-Resolution Causal Marginal Problem as an extension of the causal marginal problem where the resolution of the marginal models can differ and a common resolution must be found to represent all marginal models:

\begin{definition}[Multi-Resolution Causal Marginal Problem]\label{def:mrmp}
    Let $\mathbf{V}^*$ be the set of high-level variables, and let $\model_1, \dots, \model_n$ be SCMs together with associated mappings $\varphi_i:\vars_{\model_i} \rightarrow \vars^*$, find the space of joint causal models $\model^*$ over the variables $\vars^*$ consistent with 
    $\model_1, \dots, \model_n$.
\end{definition}

The variable sets of the marginal models no longer need to overlap, but there must exist mappings from each marginal model to a collective set of high-level variables over which an SCM can be specified. This allows for different representations of the previously overlapping variables in each model; for example, a variable $V\in \vars^*$ may be represented by multiple variables in $\model_1$ and by only one variable in $\model_2$. Additionally, this framework allows for models to have different levels of resolution for each variable; for example, a variable $V\in \vars^*$ may be define on a discretized domain in $\model_1$ and on a continuous domain in $\model_2$.

Interestingly, there is a close connection between consistency in embeddings and the marginal problem in that a set of consistent embeddings define a solution to the multi-resolution marginal problem (see proof in App.\ref{app:emb-causal-sol-proof}):

\begin{theorem}[Consistent Embeddings as Solution to the Multi-Resolution Marginal Problem]\label{thm:emb-causal-sol}
    Let $\mathbf{M}:=\{\model_1,\dots,\model_n\}$ be a set of $n$ SCMs and $\mathbf{A}:=\{\boldalpha_1,\dots,\boldalpha_n\}$ be a set of $n$ \boldalpha-embeddings such that $\boldalpha_j:\model_j\rightarrow\model'$ embeds into the same $\model'$ for all $j$. $\model'$ is a solution $\model^*$ of the multi-resolution marginal problem if for all $\boldalpha_j\in\mathbf{A}$: (i) $\boldalpha_j$ is $\mathcal{L}_i$-consistent and (ii) the set of relevant variables contains all variables $\relevant = \vars_{\model_j}$.
\end{theorem}

Notice that in the theorem we do not specify the level of consistency $\mathcal{L}_i$; \lone-consistency provides a solution to a statistical marginal problem, $\mathcal{L}_3$-consistency a solution to a causal marginal problem  as proposed by \citet{causal-marginal-problem}, and an $\mathcal{L}_2$-consistency a solution to a causal marginal problem only up to $\mathcal{L}_2$.
For an example of the use of embeddings in the multi-resolution marginal problem see the following.

\begin{example}[Multi-Resolution Marginal Problem]\label{ex:mrmp}
Recall the motivating example for causal embeddings in Ex.\ref{ex:ecosystem}-\ref{ex:finished-maps}. Models $\model_1$ and $\model_2$ (Fig.\ref{fig:fig-detailed-models}) differ in resolution for overlapping variables: whereas $\model_1$ counts subspecies of deer separately $\model_2$ only has a total count of all deer, and similarly for the predator variables. 
This does not permit for a causal marginal problem to be defined, instead this needs to be framed in the multi-resolution causal marginal problem. 
Given the emdeddings $\boldalpha_1$ and $\boldalpha_2$, we want to find a specification of an SCM $\model'$ consistent with $\model_1$ and $\model_2$. Variable maps $\varphi_i$ of the embeddings have been defined earlier, while the maps $\alpha_{V'}: \range(\varphi^{-1}(V')) \rightarrow \range(V')$ for $V' \in \relevanttarget$ can be defined simply as the sum of the pre-images. For $\boldalpha_1$ we define:
\begin{align*}
    \alpha_{\text{Deer}} &:= \text{Deer},\\
    \alpha_{\text{Humans}} &:= \text{Humans,}\\
    \alpha_{\text{Squirrels}} &:= \text{Squirrels.}
\end{align*}
and similarly for $\boldalpha_2$:
\begin{align*}
    \alpha_{\text{Deer}} &:= \text{Fallow Deer} + \text{Red Deer},\\
    \alpha_{\text{Predators}} &:= \text{Eagles} + \text{Wolves,}\\
    \alpha_{\text{Squirrels}} &:= \text{Squirrels.}
\end{align*}

It is immediate now to define a $\ltwo$-consistent SCM that is a solution to the multi-resolution marginal problem.
\end{example}

We now highlight further connections of our problem to previous marginal problems in the limit cases of reducing a multi-resolution problem to a single-resolution or using an identity embedding
(proofs in App.\ref{app:proof-reduction} and App.\ref{app:proof-id-emb-sol}):

\begin{lemma}[Multi-resolution Marginal Problem Reduction]\label{lem:reduction}
    Given a multi-resolution marginal problem, the application of \boldalpha-embeddings reduces it to a single-resolution marginal problem. 
\end{lemma}

\begin{lemma}[Identity Embeddings as Solution to the Marginal Problem]\label{lem:id-emb-sol} 
Let $\mathbf{M}:=\{\model_1,\dots,\model_n\}$ be a set of $n$ SCMs and $\mathbf{A}:=\{\boldalpha_1,\dots,\boldalpha_n\}$ be a set of $n$ \boldalpha-embeddings such that $\boldalpha_j:\model_j\rightarrow\model'$ embeds into the same $\model'$ for all $j$. $\model'$ is a solution $\model^*$ of the single-resolution marginal problem if for all $\boldalpha_j\in\mathbf{A}$: (i) $\boldalpha_j$ is \li-consistent, (ii) $\alpha_{\vars'}$ and $\varphi$ are identity maps, and (iii) the set of relevant variables contains all variables $\relevant = \vars_{\model_j}$.
\end{lemma}

\section{Merging Datasets}
Embeddings can also be applied to datasets. We propose a simple algorithm (see Alg.\ref{alg:one}), where we collect data from marginal models and we use embeddings to map the data to a single shared representation. As the marginal models are not necessarily defined on all variables, we might obtain a dataset with missing values. This constitute a case of \emph{structured missing data} due to \emph{multi-scale linkage} \citep{mitra2023learning}, and Line 6 in the algorithm calls for a data imputation method to fill in the missing values. See the following two simulated examples as an illustration:

\begin{algorithm}
\caption{Multi-resolution datasets merging}\label{alg:one}
\KwData{Datasets $\mathcal{X}_{\vars_1},...,\mathcal{X}_{\vars_n}$ from $\model_1,...,\model_n$,\\ Embeddings $\boldalpha_{\model_1},\dots,\boldalpha_{\model_n}$ s.t. $\boldalpha_i:\model_i\rightarrow\model'$}
\KwResult{Merged dataset $\mathcal{X}_{\vars'}$}
$\mathcal{X}_{\vars'} \gets [\;\;]$;\\
\For{$i \in \{0,...,n\}$}{
    $\mathcal{X}_{\vars_i'}\gets\alpha_{\vars_i}(\mathcal{X}_{\vars_i})$;\\
    $\mathcal{X}_{\vars'}\gets\left[\mathcal{X}_{\vars'},\mathcal{X}_{\vars_i'}\right]$;
}
$\mathcal{X}_{\vars'}\gets\text{impute}(\mathcal{X}_{\vars'})$
\end{algorithm}

\begin{example}[Merging datasets for increased statistical power]\label{ex:merging}
We continue Ex.\ref{ex:mrmp} and generate datasets of $2000$ and $4000$ samples for the marginal models $\model_1$ and $\model_2$, respectively (see App.\ref{app:data} for details).

After collection, data from the marginal models is transformed into the shared resolution using the embeddings $\boldalpha_1$ and $\boldalpha_2$, allowing the merging of the datasets. We estimate the distribution $\hat{P}(Deer, Squirrels)$ using the marginal and the merged datasets and computing the KL divergence between the estimation $\hat{P}$ and true distribution $P$:
\begin{align*}
    \mathcal{X}_{\model_1}&: D_{KL}(P,\hat{P}) \approx 0.34\\
    \mathcal{X}_{\model_2}&: D_{KL}(P,\hat{P}) \approx 0.77\\
    \mathcal{X}_{\model_1}+\mathcal{X}_{\model_2}&: D_{KL}(P,\hat{P}) \approx 0.22
\end{align*}
Relying on more samples, the merged dataset shows a clear improvement in the estimation of $\hat{P}$.

\begin{example}[Merging datasets to compute distributions undefined in the marginals]\label{ex:merging2}
    Suppose we want to estimate the distribution $\hat{P}(\text{Predators}, \text{Humans})$. This distribution cannot be estimated from the marginal models: $\model_1$ is not accounting for predators, while $\model_2$ ignores humans. However, by aggregating the data and imputing missing values we can estimate this quantity of interest (see App.\ref{app:imputation}).
\end{example}
\end{example}

%% file: Sections/5-Discussion.tex
\section{Conclusion}
In this work we have introduced causal embeddings as a generalization of abstractions, and extended the notion of consistency. We illustrated with examples Ex.\ref{ex:mrmp} and Ex.\ref{ex:merging} that embeddings are not just a theoretical novelty, but serve a practical use: they express a multi-resolution marginal problem, they enable its reduction to the standard marginal problem in both the statistical and causal setting, and they allow for the merging of datasets or causal models with overlapping variables at different resolutions.

Future work might extend embedding from the \boldalpha-abstraction framework to the $\tau$-abstraction (relying on the relations in \citet{schooltink2025aligning}), and explore algorithms for learning embeddings, thus enabling solving multi-resolution marginal problems and merging datasets.

%% file: Sections/A-Appendix.tex
\section{Supplementary Definitions}
\subsection{Latent Structure Projections}\label{app:latent}
Here we will introduce latent structure projections \citep{pearl1995projection}, a seminal work in graphical models with applications in graphical abstraction. In the main paper we show that through a combination of abstraction and the extension of latent structure projections to SCMs we obtain causal embeddings.

First we define latent structure models, as graphical causal models over a set of variables, some of which are unobservable:
\begin{definition}[Latent Structure Model]\label{def:latent-model}
    Given a set of causal variables $U$, a latent structure $L:=\langle D,O\rangle$ is a tuple of a DAG $D$ with vertices $U$, and a set $O\subseteq U$ of observable variables.
\end{definition}

Consider the case where some of the observable variables $O$ have become unobservable, one may want to find a new model excluding that variable in the observable set whilst preserving all conditional dependencies implied by the original. Projections allow for reasoning about such cases:
\begin{definition}[Latent Structure Projection]\label{def:latent-projection}
    A latent structure $L_{[O]}:=\langle D_{[O]},O\rangle$ is a projection of another latent structure $L$ if, and only if, the following holds:
    \begin{enumerate}
        \item Every unobservable variable of $D_{[O]}$ is a parentless common cause of exactly two non-adjacent observable variables.
        \item For every stable distribution $P$ generated by $L$ there exists a stable distribution $P'$ generated by $L_{[O]}$ such that the independencies over the variables $O$ implied by $P_{[O]}$ equal to those implied by $P'_{[O]}$.
    \end{enumerate}
\end{definition}

It has been shown that latent projections are consistent in identifiablity and preserve topological ordering. Informally, given a model $L$ and a its projection $L_{[\vars']}$ over a subset of variables $\vars' \subseteq \vars$ this implies that any causal quantities over $\vars'$ representable in $L_{[\vars']}$ are equally representable in $L$.  

\newpage
\section{Proofs}
\subsection{Proof of Lemma \ref{lem:eq-defs}}\label{app:proof-eq-defs}
\paragraph{Lemma \ref{lem:eq-defs}} (Equivalence Def.\ref{def:alpha-emb} and Def.\ref{def:alpha-emb-alt}).
\textit{Def.\ref{def:alpha-emb} of the \boldalpha-embedding in terms of projections is equivalent to Def.\ref{def:alpha-emb-alt} of the \boldalpha-embedding in terms of explicit graphical constraints. }
\begin{proof}
    In this proof we base ourselves mostly in the results of \citet{schooltink2025aligning}.
    
    First, recall conditions 1. and 2. of Def.\ref{def:alpha-emb-alt}:
    \begin{enumerate}
        \item a mediated adjacency $X'\rightsquigarrow Y'$ wrt $\relevant$ is in $\model'$ iff there exists $X \in \varphi^{-1}(X'), Y \in \varphi^{-1}(Y')$ such that there is a mediated adjacency  $X\rightsquigarrow Y$ wrt $\relevanttarget$.
        \item a mediated confounder $X'\medconfarrow Y'$ wrt $\relevant$ is in $\model'$ iff there exists $X \in \varphi^{-1}(X'), Y \in \varphi^{-1}(Y')$ such that there is a mediated confounder $X\medconfarrow Y$ wrt $\relevanttarget$.
    \end{enumerate}
    
    Importantly, (i) the $\ltwo$-distribution (in)equalities implied by a mediated edge $X\rightsquigarrow Y$ wrt $\relevant$ are equal to those implied by a directed edge $X\rightarrow Y$, given $X,Y\in \relevant$. Similarly, a mediated confounding edge $X\medconfarrow Y$ wrt $\relevant$ implies the same $\ltwo$ distribution constraints as a regular confounding edge $X\medconfarrow Y$ given $X,Y\in \relevant$. 

    Secondly, (ii) the following two conditions are known to together enforce graphically \ltwo-consistency:
    
    \begin{enumerate}
        \item an adjacency $X'\rightarrow Y'$ is in $\model'$ iff there exists $X \in \varphi^{-1}(X'), Y \in \varphi^{-1}(Y')$ such that there is a mediated adjacency  $X\rightsquigarrow Y$ wrt $\relevanttarget$.
        \item a confounder $X'\confarrow Y'$  is in $\model'$ iff there exists $X \in \varphi^{-1}(X'), Y \in \varphi^{-1}(Y')$ such that there is a mediated confounder $X\medconfarrow Y$ wrt $\relevanttarget$.
    \end{enumerate}
    
    Therefore, by a combination of (i) and (ii) it follows that conditions 1. and 2. from Def.\ref{def:alpha-emb-alt} must entail graphical $\ltwo$ consistency.

    Similarly, Def.\ref{def:alpha-emb} is necessarily graphically $\ltwo$ consistent, as the CDAGs are known to be graphically $\ltwo$-consistent \citep{anand2023causal} and the projections by definition preserve \ltwo-distributional (in)equalities.

    Thus, both Def.\ref{def:alpha-emb} and Def.\ref{def:alpha-emb-alt} define an \boldalpha-embedding to be a non-surjective \alphabs{} such that \boldalpha{} is graphically $\ltwo$-consistent.
\end{proof}

\subsection{Proof of Theorem \ref{thm:graph-implies-func}}\label{app:graph-implies-func}
\paragraph{Theorem 3}(Graphical $\ltwo$-consistent map $\Rightarrow$ Functional $\ltwo$-consistent embedding)\textbf{.}
\textit{Given a graphically \ltwo-consistent $\varphi:\relevant\rightarrow\relevanttarget$ with $\relevant\subseteq\vars_{\model}$ and $\relevanttarget\subseteq\vars_{\model'}$ there exists a specification of $\model'$ and an \boldalpha-embedding $\boldalpha:\model\rightarrow\model'$ with variable map $\varphi$ s.t. $\boldalpha$ is functionally \ltwo-consistent.}

\begin{proof}
    Let $\model$ be a SCM over the endogenous variables $\vars_\model$, $\model'$ an SCM over the endogenous variables $\vars_{\model'}$, and $\varphi:\relevant\rightarrow\relevanttarget$ such that $\relevant\subseteq\vars_{\model}$ and $\relevanttarget\subseteq\vars_{\model'}$.
    
    Note that $\varphi$ is graphically $\ltwo$ consistent and thus the causal graph $G_\relevant$ over $\relevant$ to the causal graph $G_\relevanttarget$ over $\relevanttarget$ is necessarily a CDAG \citep{anand2023causal, schooltink2025aligning}. Therefore, given an SCM $\model_{[\relevant]}$ over $\relevant$ and an SCM $\model'_{[\relevanttarget]}$ over $\relevanttarget$, there exists a specification of $\model'_{[\relevanttarget]}$ such that $\model'_{[\relevanttarget]}$ is functionally \ltwo-consistent with $\model_{[\relevant]}$. This leaves to show that (i) there is a model specification of $\model_{[\relevant]}$ encoding the same distributions as $\model$ and (ii) there is a specification of $\model'$ encoding the same distributions as $\model'_{[\relevanttarget]}$.

    (i) is given by the results of \citet{schooltink2025aligning} Thm.23, and (ii) is satisfied by the following construction of functions $\mathcal{F}_{\model}$:

    We distinguish three types of endogenous variables in $\vars_{\model'}$ for which we specify a function respecting the causal graph:
    \begin{enumerate}
        \item Variables $V$ s.t. $V\notin \relevanttarget$ and $\forall S\in \relevanttarget, V\notin An(S)$. As these variables do not influence distributions of interest, arbitrary functions can be defined.
        \item Variables $V$ s.t. $V\notin \relevanttarget$ and $\exists S\in \relevanttarget, V\in An(S)$. These variables that do influence some $S\in\relevanttarget$, the function determining their value is taking the cartesian product of the parent variables including exogenous variables.
        \item Variables $V$ s.t. $V\in \relevanttarget$. These variables must produce same distributions as given in $\model'_{[\relevanttarget]}$. By preservation of causal paths through projection the parents the function $f_V$ must have access to the values of all its parents in $\model_{[\relevanttarget]}'$ either because the variable is in the signature of $f_V$, or if there is a mediated adjacency wrt $\relevanttarget$ there exist a variable in the signature which carries its value (since all intermediate variables take cartesian products of the parents). Thus allowing the specification of the function $f_V$ in $\model'$ to be equal to that in $\model_{[\relevanttarget]}$.  
    \end{enumerate}
    Specifically, this construction additionally allows for preservation of mediated confounding, as exogenous parents are also preserved in intermediate variables. Thus it follows that this model $\model'$ can represent the exact same causal effects as $\model_{[\relevanttarget]}$, and generate the same distributions.

    Showing that given an SCM $\model$ over the endogenous variables $\vars_\model$,  an SCM $\model'$ over the endogenous variables $\vars_{\model'}$, and $\varphi:\relevant\rightarrow\relevanttarget$ such that $\relevant\subseteq\vars_{\model}$ and $\relevanttarget\subseteq\vars_{\model'}$, we can always construct a specification for $\model'$ such that there exists a functionally $\ltwo$-consistent \boldalpha-embedding $\boldalpha:\model\rightarrow\model'$.
\end{proof}

\subsection{Proof of Theorem \ref{thm:func-not-implies-graph}}\label{app:func-not-implies-graph-proof}
\paragraph{Theorem 4}(Functional $\ltwo$-consistent non-surjective abstraction $\not\Rightarrow$ embedding)\textbf{.}
\textit{A non-surjective \ltwo-consistent \boldalpha-abstraction does not necessarily imply an \boldalpha-embedding.}
\begin{proof}
    For this proof we will show that one can create a functionally $\ltwo$-consistent non-surjective abstraction, that is not graphically $\ltwo$-consistent. Thus by Thm.\ref{thm:emd-l2} the proposed abstraction is not an \boldalpha-embedding.

    We will define a \boldalpha-abstraction $\boldalpha:\model\rightarrow\model'$. Let $\model$ and $\model'$ be SCMs with the following specification:
    
    \begin{minipage}{.45\linewidth}
        \begin{align*}
            &\model:\\
            \vars&:=\{X,Y,Z\}, \text{ with }\begin{cases}
                 \range(X)&=\{0,2,4\}\\ \range(Y)&=\{0,1\}\\ \range(Z)&=\{0,1,2,3,4,5\}
            \end{cases}\\
            \exos&:=\{U_X, U_Y\}\\
            \funcs&:=\begin{cases}
                f_X(U_X)&=U_X\\
                f_Y(U_Y)&=U_Y\\
                f_Z(X,Y)&=X+Y
            \end{cases}\\
            \dists&:=\begin{cases}
                U_X &\sim \mathcal{U}\{0,2,4\}\\
                U_Y &\sim B(0.5)
            \end{cases}
        \end{align*}
    \end{minipage}
    \hfill
    \vline
    \hfill
    \begin{minipage}{.45\linewidth}
        \begin{align*}
            &\model':\\
            \vars&:=\{X',Y',Z'\}, \text{ with }\begin{cases}
                 \range(X')&=\{0,2,4\}\\ \range(Y)&=\{0,1\}\\ \range(Z)&=\{0,1\}
            \end{cases}\\
            \exos&:=\{U_{X'}, U_{Y'}\}\\
            \funcs&:=\begin{cases}
                f_{X'}(U_{X'})&=U_{X'}\\
                f_{Y'}(U_{Y'})&=U_{Y'}\\
                f_{Z'}(Y')&=Y'
            \end{cases}\\
            \dists&:=\begin{cases}
                U_{X'} &\sim \mathcal{U}\{0,2,4\}\\
                U_{Y'} &\sim B(0.5)
            \end{cases}
        \end{align*}
    \end{minipage}

    \vspace{1em}
    Then define $\boldalpha$ with $\relevant=\vars_{\model}$ and $\relevanttarget=\vars_{\model'}$ by the maps:
    \begin{align*}
        \varphi&:=\begin{cases}
            X&\mapsto X'\\
            Y&\mapsto Y'\\
            Z&\mapsto Z'
        \end{cases}, &
        \alpha_{X'}&:=\begin{cases}
            0&\mapsto 0\\
            2&\mapsto 2\\
            4&\mapsto 4
        \end{cases},&
        \alpha_{Y'}&:=\begin{cases}
            0&\mapsto 0\\
            1&\mapsto 1
        \end{cases},&
        \alpha_{Z'}&:=\begin{cases}
            0&\mapsto 0\\
            1&\mapsto 1\\
            2&\mapsto 0\\
            3&\mapsto 1\\
            4&\mapsto 0\\
            5&\mapsto 1
        \end{cases}.\\
    \end{align*}

    One can easily verify $\boldalpha$ is functionally \ltwo-consistent. Now we draw the DAGs entailed by $\model$ and $\model'$:
    \begin{center}
        \begin{tikzpicture}
            \node (M) at (1,0.75) {$\model:$};
            \node (X) at (0,0) {$X$};
            \node (Y) at (0,-2) {$Y$};
            \node (Z) at (2,-1) {$Z$};

            \draw[->] (X) -- (Z);
            \draw[->] (Y) -- (Z);
        \end{tikzpicture}
        \hspace{5em}
        \begin{tikzpicture}
            \node (M) at (1,0.75) {$\model':$};
            \node (X) at (0,0) {$X'$};
            \node (Y) at (0,-2) {$Y'$};
            \node (Z) at (2,-1) {$Z'$};

            \draw[->] (Y) -- (Z);
        \end{tikzpicture}
    \end{center}
    Notice that the edge $X'\rightarrow Z'$ is absent as $X'$ is not in the signature of the function $f_{Z'}$. Thus the algebraic constraint $P(Z'\vert do(X')) = P(Z'\vert do(\emptyset))$ is in $\mathcal{G}^{\ltwo}(G_{[\relevanttarget]})$. In turn, $P(Z\vert do(X)) = P(Z\vert do(\emptyset))$ is in $\mathcal{G}^{\ltwo}(G^{-1}_{[\relevanttarget]})$. However the graph of $\model$ implies the constraint $P(Z\vert do(X)) \neq P(Z\vert do(\emptyset))$, and thus $$\mathcal{G}^{\ltwo}\left(G_{\left[\relevanttarget\right]}^{-1}\right) \not\subseteq \mathcal{G}^{\ltwo}\left(G^{ }_{\left[\relevant\right]}\right).$$

    Showing a functional $\ltwo$-consistent non-surjective abstraction does not necessarily imply graphical \ltwo-consistency, and consequently not an $\boldalpha$-embedding. 
\end{proof}

\subsection{Proof of Theorem \ref{thm:emb-causal-sol}}\label{app:emb-causal-sol-proof}
\paragraph{Theorem 5}(Consistent Embeddings as Solution to the Multi-Resolution Causal Marginal Problem)\textbf{.}
\textit{Let $\mathbf{M}:=\{\model_1,\dots,\model_n\}$ be a set of $n$ SCMs and $\mathbf{A}:=\{\boldalpha_1,\dots,\boldalpha_n\}$ be a set of $n$ \boldalpha-embeddings such that $\boldalpha_j:\model_j\rightarrow\model'$ embeds into the same $\model'$ for all $j$. $\model'$ is a solution $\model^*$ of the multi-resolution marginal problem if for all $\boldalpha_j\in\mathbf{A}$: (i) $\boldalpha_j$ is $\mathcal{L}_i$-consistent and (ii) the set of relevant variables contains all variables $\relevant = \vars_{\model_j}$.}

\begin{proof}
    A solution to the multi-resolution causal marginal problem requires 
    \begin{itemize}
        \item[(i)] a mapping $\varphi_j:\vars_{\model_j}\rightarrow\vars^*$ from the variables of the low-level models into a shared set of variables $\vars^*$,
        \item[(ii)] a model $\model^*$ over the set of shared variables $V^*$,
        \item[(iii)] the model $\model^*$ to be consistent with all $\model_j\in\mathbf{M}$.
    \end{itemize}

    Requirement (i) is satisfied, since all model $\model_j\in\mathbf{M}$ embed into $\model'$ and thus have a mapping $\varphi_j:\vars_{\model_j}\rightarrow\vars_{\model'}$ making $\vars_{\model'}$ the set of shared variables $\vars^*$. (ii) is then immediately satisfied as $\model'$ is the SCM over $\vars^*$, thus $\model' =\model^*$. Finally, (iii) is satisfied by all embeddings being consistent.

    Therefore, $\model'$ is a solution $\model^*$ to the multi-resolution marginal problem at the resolution of $\vars^*$.
\end{proof}

\subsection{Proof of Lemma \ref{lem:reduction}}\label{app:proof-reduction}
\paragraph{Lemma 6}(Multi-resolution Marginal Problem Reduction)
\textit{Given a multi-resolution marginal problem, the application of \boldalpha-embeddings reduce it to a single-resolution causal marginal problem. }
\begin{proof}
    Let $\mathbf{M}:=\{\model_1,\dots,\model_n\}$ be a set of $n$ SCMs and $\mathbf{A}:=\{\boldalpha_1,\dots,\boldalpha_n\}$ be a set of $n$ \boldalpha-embeddings such that $\boldalpha_i:\model_i\rightarrow\model'$ embeds into the same $\model'$ for all $i$. For all $\model_i\in\mathbf{M}$ the application $\boldalpha_i(\model_i)$ applies the maps $\varphi:\relevant\rightarrow\relevanttarget$ and $\alpha_{V'}:\range(\varphi^{-1}(V'))\rightarrow\range(V')$ for all $V'\in\relevanttarget$. Since $\relevanttarget\subseteq\vars_{\model'}$ and $\vars_{\model'}$ is the same for all embeddings $\boldalpha_i$, the transformed models $\boldalpha_i(\model_i)$ necessarily have equal representations and resolutions for any overlapping variables. Thus reducing to the single-resolution causal marginal problem of finding a model $\model'$ consistent with marginal models $\boldalpha_1(\model_1),...,\boldalpha_n(\model_n)$ with possible non-empty overlapping variables.
\end{proof}

\subsection{Proof of Lemma \ref{lem:id-emb-sol}}\label{app:proof-id-emb-sol}
\paragraph{Lemma 7}(Identity Embeddings as Solution to the Marginal Problem)
\textit{Let $\mathbf{M}:=\{\model_1,\dots,\model_n\}$ be a set of $n$ SCMs and $\mathbf{A}:=\{\boldalpha_1,\dots,\boldalpha_n\}$ be a set of $n$ \boldalpha-embeddings such that $\boldalpha_j:\model_j\rightarrow\model'$ embeds into the same $\model'$ for all $j$. $\model'$ is a solution $\model^*$ of the single-resolution marginal problem if for all $\boldalpha_j\in\mathbf{A}$: (i) $\boldalpha_j$ is \li-consistent, (ii) $\alpha_{\vars'}$ and $\varphi$ are identity maps, and (iii) the set of relevant variables contains all variables $\relevant = \vars_{\model_j}$.}
\begin{proof}
    First note that an $\boldalpha$-embedding with $\relevant = \vars_{\model_j}$ and $\alpha_{\vars'}$, $\varphi$ identity maps does not alter the model if applied to it: $\boldalpha(\model) = \model$. Then as a consequence of Lem.\ref{lem:reduction} it follows that if all $\boldalpha_j\in\textbf{A}$ are such identity embeddings the problem is the same as finding a specification $\model'$ consistent with models $\model_1,...,\model_n$ as in the single-resolution marginal problem.
\end{proof}

\newpage
\section{Examples}
\subsection{Non-Uniqueness of Distributions given Consistent Embeddings -- Remark \ref{rem:1}}\label{app:rem-1}
Assume we have two SCMs $\model_1,\model_2$ with respective graphs $X\rightarrow Y$ and $Y\rightarrow Z$, with all variables binary: $X,Y,Z\in\{0,1\}$. We then define an embedding into an SCM $\model'$ over the variables $X,Y,Z$ s.t. all maps of the embeddings are identity maps, as shown in the following diagram:
\begin{center}
    \begin{tikzpicture}
        \node (M1) at (-.5,2.75) {$\model_1$:};
        \node (X1) at (-1.5,2) {$X$};
        \node (Y1) at (.5,2) {$Y$};
        \draw[->] (X1) -- (Y1);
        \node (M2) at (3.5,2.75) {$\model_2$:};
        \node (Y2) at (2.5,2) {$Y$};
        \node (Z2) at (4.5,2) {$Z$};
        \draw[->] (Y2) -- (Z2);
    
        \node (MP) at (-1,-.2) {$\model':$};
        \node (X) at (0,0) {$X$};
        \node (Y) at (1.5,0) {$Y$};
        \node (Z) at (3,0) {$Z$};
    
        \draw[->] (X) -- (Y);
        \draw[->] (X) to[in=270-45, out=-45] (Z);
        \draw[->] (Y) -- (Z);
    
        \node[orange] (a1) at (-1.75,1.2) {$\boldalpha_1$:};
        \node[blue]   (a2) at (1.5,1.2) {$\boldalpha_2$:};
        \draw[|->, orange] (X1) -- node[midway,left] {id.} (X);
        \draw[|->, orange] (Y1) -- node[midway,left] {id.} (Y);
        \draw[|->, blue]   (Y2) -- node[midway,right]{id.} (Y);
        \draw[|->, blue]   (Z2) -- node[midway,right]{id.} (Z);
    \end{tikzpicture}
\end{center}

We will show through a concrete example that \lone-consistency of both embeddings does not mean all $\lone$ distributions of $\model'$ are uniquely determined: First, if the embeddings {\color{orange}$\boldalpha_1$}, {\color{blue}$\boldalpha_2$} are \lone-consistent, and the maps $\varphi:\vars\rightarrow\vars'$ and $\alpha_{\vars'}:\range(\vars)\rightarrow\range(\vars')$ are identities all $\lone$ distributions of $\model_1$ and $\model_2$ must be exactly equal to those distributions in $\model'$. So the following distributions are fixed for $\model'$:
$$P(X), P(Y), P(Z), P(X\vert Y), P(Y\vert X), P(Y\vert Z), P(Z\vert Y), P(Y,X), P(Y,Z).$$

In the following we will show that this not enforce uniqueness of the $\lone$ distribution $P(Z\vert X,Y)$. First we set the following distributions that are enforced by consistency of the embeddings:
\begin{align*}
    P(X)&:= \begin{tabular}{R|C|C}
        X: & 0 & 1\\\hline
        P(X): & 0.4 & 0.6
    \end{tabular}\\\\
    P(Y\vert X)&:= \begin{tabular}{R|C|C}
        Y: & 0 & 1\\\hline
        P(Y\vert X=0): & 0.7 & 0.3\\
        P(Y\vert X=1): & 0.4 & 0.6
    \end{tabular}& P(Y)&:=\begin{tabular}{R|C|C}
        Y: & 0 & 1\\\hline
        P(Y): & 0.52 & 0.48
    \end{tabular}\\\\
    P(Z\vert Y)&:= \begin{tabular}{R|C|C}
        Z: & 0 & 1\\\hline
        P(Z\vert Y=0): & 0.6 & 0.4\\
        P(Z\vert Y=1): & 0.3 & 0.7
    \end{tabular}& P(Z)&:=\begin{tabular}{R|C|C}
        Z: & 0 & 1\\\hline
        P(Z): & 0.456 & 0.544
    \end{tabular}
\end{align*}
Now we construct any number of distributions for $P(Z\vert Y,X)$ as long as $\sum_{x\in \range(X)}P(Z\vert Y,X=x)P(X=x) = P(Z\vert Y)$ for example the following two:
\begin{align*}
    P_1(Z\vert Y, X)&:= \begin{tabular}{R|C|C}
        Z: & 0 & 1\\\hline
        P_1(Z\vert Y=0,X=0): & 0.6 & 0.4\\
        P_1(Z\vert Y=0,X=1): & 0.6 & 0.4\\
        P_1(Z\vert Y=1,X=0): & 0.3 & 0.7\\
        P_1(Z\vert Y=1,X=1): & 0.3 & 0.7\\
    \end{tabular},
    & P_2(Z\vert Y, X)&:= \begin{tabular}{R|C|C}
        Z: & 0 & 1\\\hline
        P_2(Z\vert Y=0,X=0): & 0.4 & 0.6\\
        P_2(Z\vert Y=0,X=1): & 0.7\overline{33} & 0.2\overline{66}\\
        P_2(Z\vert Y=1,X=0): & 0.2 & 0.8\\
        P_2(Z\vert Y=1,X=1): & 0.3\overline{66} & 0.6\overline{33}\\
    \end{tabular}.
\end{align*}
Thus showing $\lone$-consistency of all embeddings does not uniquely determine all $\lone$ distributions in $\model'$.
\clearpage

\subsection{Data Generation -- Example \ref{ex:merging}}\label{app:data}
For Ex.\ref{ex:merging} we generate data for both marginal models under the assumption they both measure the same system. As such, for data generation one large causal model is defined over which both models obtain samples. First, recall the marginal models:

\begin{center}
    \resizebox{.6\linewidth}{!}{
    \begin{tikzpicture}[x=-1cm]
            \node (model1) at (-2.1,-1.75) {$\model_2:$};
            \node (Red Deer) at (-1.1,-2.5) {Red Deer};
            \node (Fallow Deer) at (-1.1,-4) {Fallow Deer};
            \node (Rabbits3) at (-1.1,-5.5) {Squirrels};
            \node (Wolves) at (-3.5,-3) {Wolves};
            \node (Eagles) at (-3.5,-5) {Eagles};
    
            \draw[->] (Wolves) -> (Red Deer);
            \draw[->] (Wolves) -> (Fallow Deer.north east);
            \draw[->] (Eagles) -> (Fallow Deer.south east);
            \draw[->] (Eagles) -> (Rabbits3);
            \draw[->] (Fallow Deer) -> (Rabbits3);
            \draw[<->, dashed] (Red Deer) to[in=160, out=-160] (Fallow Deer);
            \draw[<->, dashed] (Fallow Deer) to[in=160, out=-160] (Rabbits3);
            \draw[<->, dashed] (Red Deer) to[in=160, out=-160] (Rabbits3);
    
            \node (model2) at (6,-2.25) {$\model_1:$};
            \node (Rabbits2) at (6.5,-3) {Squirrels};
            \node (Deer2) at (4.6,-4.5) {Deer};
            \node (Humans2) at (8,-4.5) {Humans};
            \node (Berries) at (4.7,-3.2) {Berries};
    
            \draw[->] (Humans2) -> (Deer2);
            \draw[->] (Deer2) -> (Rabbits2);
            \draw[->] (Berries) -> (Deer2);
            \draw[->] (Berries) -> (Rabbits2);
            \draw[->] (Humans2) -> (Rabbits2);
        \end{tikzpicture}
    }
\end{center}

We construct the data generation model as the following SCM. Important to note is that the functions and distributions are mostly chosen to generate reasonable looking distributions, they are not based in any real-world data.
\begin{align*}
    \vars &:= \{\text{Wolves, Eagles, Fallow Deer, Red Deer, Squirrels, Humans, Berries}\}\\
    \exos &:= \{U_{\text{Wolves}},U_{\text{Eagles}},U_{\text{Humans}},U_{\text{Berries}}\}\\
    \funcs &:= \begin{cases}
                F_{\text{Wolves}}(U_{\text{Wolves}}) &= 100 \times \max(U_{\text{Wolves}},0.1)\\
                F_{\text{Eagles}}(U_{\text{Eagles}}) &= 10 \times \max(U_{\text{Eagles}},0.1)\\
                F_{\text{Humans}}(U_{\text{Humans}}) &= 15 \times \max(U_{\text{Humans}},0.1)\\
                F_{\text{Berries}}(U_{\text{Berries}}) &= \max(U_{\text{Berries}},0.1)\\
                F_{\text{Fallow Deer}}(\text{Berries, Wolves, Eagles, Humans}) &= \max(300 \times \text{Berries} - \text{Wolves} - 2 \times\text{Eagles} - 3\times \text{Humans}, 0)\\
                F_{\text{Red Deer}}(\text{Berries, Wolves, Humans}) &= \max(200 \times \text{Berries} - \text{Wolves} - 3\times \text{Humans}, 0)\\
                F_{\text{Squirrels}}(\text{Berries, Eagles, Fallow Deer, Humans}) &= \max(200 \times \text{Berries} - 5\times\text{Eagles} - 4\times \text{Humans} \\& \;\;\;\;\;\;\;\;\;\;\;\;- \frac{1}{2}\times \text{Fallow Deer}, 0)
               \end{cases}\\
    \dists &:= \begin{cases}
                P(U_\text{Wolves}) &\sim N(1,0.20)\\
                P(U_\text{Eagles}) &\sim N(1,0.15)\\
                P(U_\text{Humans}) &\sim N(1,0.25)\\
                P(U_\text{Berries})&\sim N(1,0.25)\\
              \end{cases}
\end{align*}
Giving rise to the ground truth graph:
\begin{center}
    \begin{tikzpicture} [x=-1cm]
        \node (W) at (0,0) {Wolves};
        \node (E) at (0,-1.5) {Eagles};
        \node (RD) at (3,.75) {Red Deer};
        \node (FD) at (3,-.75) {Fallow Deer};
        \node (R) at (3,-2.25) {Squirrels};
        \node (B) at (6,.5) {Berries};
        \node (H) at (6,-2) {Humans};

        \draw[->] (W) -- (RD);
        \draw[->] (W) -- (FD);
        \draw[->] (E) -- (FD);
        \draw[->] (E) -- (R);
        \draw[->] (FD)-- (R);
        \draw[->] (B) -- (RD);
        \draw[->] (B) -- (FD);
        \draw[->] (B) -- (R);
        \draw[->] (H) -- (RD);
        \draw[->] (H) -- (FD);
        \draw[->] (H) -- (R);
    \end{tikzpicture}
\end{center}

For the purpose of avoiding fractional animals, we round the outputs of the functions up to integers.

We generate three datasets by sampling the exogenous variables $\exos$ and using the functions $\funcs$ to find the values of the endogenous variables $\vars$. First, for the dataset for model $\model_1$, we generate 2000 samples, and omit the Berries and Humans variables. For the dataset for model $\model_2$, we generate 4000 samples, assign Deer=Fallow Deer + Red Deer, and drop the variables Wolves and Eagles. And finally, we generate 100000 samples for an ground truth dataset to compare against. All datasets are independently sampled, so there are no samples shared between the datasets. The implementation is available online\footnote{\smaller\url{\gitlink}}.

\subsection{Embeddings for Marginal Problem -- Example \ref{ex:merging2}}\label{app:imputation}
Continuing App.\ref{app:data}, we have two datasets with samples for $\model_1$ and $\model_2$, respectively. We illustrate by example how one can use embeddings to merge the datasets, and through imputation estimate distributions not represented in the marginal models.

We will follow Alg.\ref{alg:one} to construct a single dataset with missing values imputed. For the imputation step we choose a KNN-imputer with $K=2$. The implementation is provided online\footnote{\smaller\url{\gitlink}}. The application of Alg.\ref{alg:one} provides one dataset of 6000 samples, with missing values imputed. Thus allowing for the estimation of $\hat{P}(\text{Humans},\text{Predators})$, a distribution not available in either $\model_1$ or $\model_2$. We provide a visual comparison between the estimation $\hat{P}(\text{Humans},\text{Predators})$ and the true distribution from the ground truth dataset in Fig.\ref{fig:plots}.

\begin{figure}[ht!]
    \centering
    \includegraphics[width=0.45\linewidth]{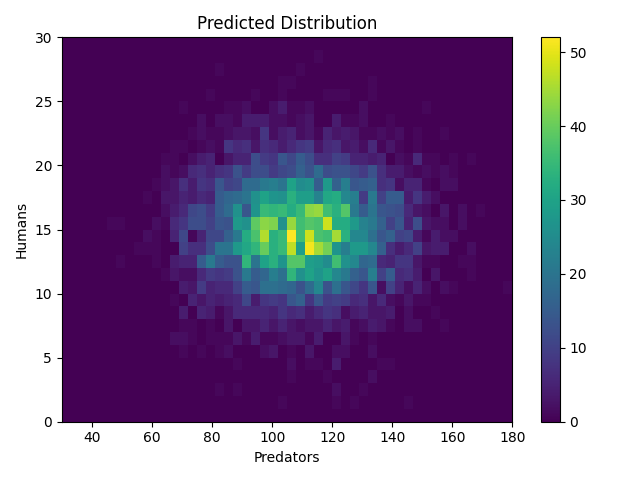}
    \includegraphics[width=0.45\linewidth]{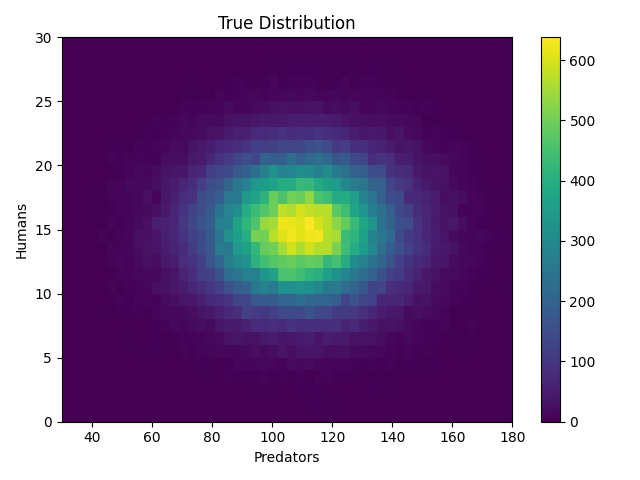}
    \caption{A visual comparison between the predicted estimation $\hat{P}(\text{Humans},\text{Predators})$ and the evaluation distribution $P(\text{Humans},\text{Predators})$. Imputation allows for approximation of distributions otherwise not available.}
    \label{fig:plots}
\end{figure}

Note we choose the KNN-imputer to approximate a solution to the statistical marginal problem in this case. This is a naive and illustrative choice, we do not claim the KNN-imputer to be an optimal approximation. A study of solutions to the marginal problem are out of scope for this work.